\documentclass{article} 
\usepackage{iclr2026_conference,times}
\usepackage{amsmath,amssymb,amsthm,mathtools}
\usepackage{booktabs}
\usepackage{microtype}
\usepackage{hyperref}
\usepackage[nameinlink,noabbrev]{cleveref}
\usepackage[ruled,linesnumbered]{algorithm2e}
\makeatletter
\renewcommand{\nl}{%
  \refstepcounter{AlgoLine}%
  \hbox to 1.8em{\scriptsize\arabic{AlgoLine}\hfil}%
}
\makeatother

\usepackage{amsmath,amsfonts,bm}

\def\eqref#1{equation~\ref{#1}}

\def\1{\bm{1}}

\def\rvv{{\mathbf{v}}}

\def\ervv{{\textnormal{v}}}

\def\vf{{\bm{f}}}

\def\vh{{\bm{h}}}

\def\vu{{\bm{u}}}
\def\vv{{\bm{v}}}
\def\vw{{\bm{w}}}

\def\vy{{\bm{y}}}

\def\evu{{u}}
\def\evv{{v}}

\def\mA{{\bm{A}}}

\def\mW{{\bm{W}}}
\def\mX{{\bm{X}}}
\def\mY{{\bm{Y}}}

\DeclareMathAlphabet{\mathsfit}{\encodingdefault}{\sfdefault}{m}{sl}
\SetMathAlphabet{\mathsfit}{bold}{\encodingdefault}{\sfdefault}{bx}{n}

\DeclareMathOperator{\sign}{sign}

\usepackage{url}

\numberwithin{equation}{section}

\newtheorem{theorem}{Theorem}[section]
\newtheorem{lemma}[theorem]{Lemma}

\newtheorem{proposition}[theorem]{Proposition}

\newtheorem{remark}[theorem]{Remark}
\newtheorem{assumption}[theorem]{Assumption}

\newtheorem{innercustomlemma}{Lemma}
\newenvironment{customlemma}[1]
  {\renewcommand{\theinnercustomlemma}{#1}\begin{innercustomlemma}}
  {\end{innercustomlemma}}

\newcommand{\vvarphi}{\bm{\varphi}}
\newcommand{\Prob}{\mathbb{P}}
\newcommand{\Expect}{\mathbb{E}}
\newcommand{\Wcal}{\mathcal{W}}
\newcommand{\Ycal}{\mathcal{Y}}
\newcommand{\Vcal}{\mathcal{V}}
\newcommand{\Bcal}{\mathcal{B}}
\newcommand{\Wfrak}{\mathfrak{W}}

\newcommand{\od}{\odot}

\newcommand{\wSig}{w'_{\Sigma}}
\newcommand{\PNOne}{\left\{ -1, +1 \right\}}

\title{Factorized AdaBoost.MH Achieves the Same Convergence Rate as AdaBoost.MH}

\author{\parbox{\textwidth}{\centering
Xin Zou$^{1}$, Jingyuan Xu$^{2}$\thanks{Corresponding author: Jingyuan Xu (\texttt{jingyuanxu777@gmail.com}).}\\
$^{1}$Independent Researcher\\
$^{2}$School of Computer Science, Wuhan University, Wuhan, China
}}

\iclrfinalcopy 
\begin{document}

\maketitle
\fancyhead[L]{Preprint. Work in progress.}

\begin{abstract}
AdaBoost.MH reduces multi-class classification to a collection of binary subproblems and enjoys the classical boosting-type convergence guarantee under a weak learning condition. A more structured variant, Factorized AdaBoost.MH, uses base classifiers of the form $\vh(x)=\alpha \vv \varphi(x)$, where a single binary classifier $\varphi$ is shared across all classes and the label dependence is carried by a vote vector $\vv\in\{\pm1\}^K$. This factorization is algorithmically attractive and achieves better performance in practice, but its convergence depends on whether one can always choose a vote vector with sufficiently large induced binary weight mass. Previous work resolved this question with a lower bound $\max\{1/n,1/\sqrt{2K}\}$, which still leaves a dimension-dependent slowdown relative to the original AdaBoost.MH analysis. In this paper, we sharpen this combinatorial step and provide an exact characterization. For the minimax quantity $\mathfrak{W}_{n,K}$ governing the factorized edge, we prove $\mathfrak{W}_{n,K} = C_{\min\{n+1,K\}}$, where $C_q=1$ for $q=1$, $C_q=q/(3q-4)$ for even $q\ge2$, and $C_q=(q+1)/(3q-1)$ for odd $q\ge2$. Since $C_q\downarrow 1/3$, our bounds show that $\mathfrak{W}_{n,K}=\Theta(1)$ uniformly over $n$ and $K$. Consequently, Factorized AdaBoost.MH achieves the same boosting-type convergence rate as AdaBoost.MH up to a universal constant factor, removing the previously suggested additional dependence on $n$ or $K$ in the number of boosting rounds.
\end{abstract}

\section{Introduction}

Boosting is a central paradigm in supervised learning because it explains how many simple prediction rules can be combined into a highly accurate classifier and has inspired a lot on theoretical analysis and algorithm design in supervised learning \citep{DBLP:journals/jmlr/Hanneke16,DBLP:conf/colt/MontasserHS19}. A boosting algorithm adaptively reweights the training sample, asks for a rule that is only slightly better than random guessing under the current weights, and aggregates the returned rules into a strong predictor. This weak-to-strong principle is important both algorithmically and theoretically: it gives practical procedures that are easy to implement, and it provides a clean route from a weak learning condition to explicit training-error and margin guarantees. In the binary case, the classical AdaBoost analysis shows that, whenever every boosting round achieves a positive edge, i.e., has a better performance than random guessing, the exponential loss decreases geometrically \citep{10.7551/mitpress/boosting-book}.

Many modern classification tasks have more than two classes, so it is natural to ask whether the same weak-to-strong mechanism remains effective when the predictor must distinguish among \(K\) labels rather than two. AdaBoost.MH extends AdaBoost by encoding each labeled example as a vector of \(K\) binary label coordinates and by maintaining a weight matrix over example--label pairs \citep{DBLP:journals/ml/SchapireS99}. This construction is a foundational multi-class boosting framework: it turns a multi-class problem into a structured collection of binary decisions and minimizes a weighted Hamming loss through an
exponential surrogate. In the standard analysis, the base learner can choose a separate binary decision for each class coordinate, so the empirical weak learning assumption can be applied almost coordinate by coordinate.

The original AdaBoost.MH algorithm represents the output of the weak learner as
\begin{equation*}
  \vh(x) = \alpha \vvarphi(x),
\end{equation*}
where $\alpha$ is a scalar coefficient and $\vvarphi(x)=[\varphi_1(x),\ldots,\varphi_K(x)]^\top$ maps $\mathcal{X}$ to $\{\pm1\}^K$. This form treats the \(K\) class coordinates as separate binary subproblems. While this independence is convenient for the convergence proof, it may be unnecessarily loose as a modeling assumption: multi-class labels are mutually exclusive, and useful weak rules often capture relations among several classes at once. This motivates more structured weak hypotheses that preserve the boosting philosophy while coupling the class coordinates. In particular, \citet{DBLP:journals/corr/Kegl13} studied the factorized classifier
\[
  \vh(x)=\alpha \vv \varphi(x),
\]
where $\alpha$ is a scalar coefficient, $\vv\in\{\pm1\}^K$ is an input-independent vote vector over classes, and
$\varphi:\mathcal{X}\to\{\pm1\}$ is a single binary classifier shared by all labels. This form can be viewed as decomposing a vector-valued weak rule into a class-side vote vector and an input-side binary split. It avoids training \(K\) unrelated weak rules at each boosting round and can encode multi-class structure through the vote vector. Empirically, such factorized weak classifiers are attractive and achieve better performance \citep{DBLP:journals/corr/Kegl13}, but their theoretical analysis is more delicate: once $\vv$ is fixed, all label coordinates must be compressed into one induced binary problem on the original examples, so the standard coordinate-wise proof for AdaBoost.MH no longer applies.

The key quantity in this compression is the total induced weight
\[
  \wSig(\mW,\mY,\vv)
  =
  \left\|(\mW\odot\mY)\vv\right\|_1,
\]
where $\mW$ is the current normalized weight matrix and $\mY$ is the signed one-hot label matrix. If this quantity is bounded below, then the standard binary weak-learning assumption can be applied to the induced pseudo-labeled sample, yielding a positive factorized edge and hence a decrease in the exponential loss. The central obstruction is therefore purely combinatorial: for every possible weight matrix and label matrix, must there exist a vote vector $\vv$ for which $\wSig(\mW,\mY,\vv)$ is bounded away from zero? \citet{DBLP:conf/colt/Kegl14} posed this as an open problem.

\citet{DBLP:conf/nips/ZouZXL24} answered this question positively by proving a
lower bound for the associated minimax value
\[
  \Wfrak_{n,K}
  =
  \min_{\mW,\mY}
  \max_{\vv\in\{\pm1\}^K}
  \left\|(\mW\odot\mY)\vv\right\|_1.
\]
Their result shows that
$\Wfrak_{n,K}\ge \max\{1/n,1/\sqrt{2K}\}$, which is enough to establish boosting-type convergence for Factorized AdaBoost.MH whenever $n$ and $K$ are not simultaneously large. Nevertheless, the bound leaves open the possibility that factorization intrinsically costs a factor depending on the number of examples or classes. Indeed, plugging this bound into the standard exponential loss analysis gives a sufficient number of boosting rounds larger than that of AdaBoost.MH by a factor of order $\min\{n^2,K\}$.

This paper shows that this dimension-dependent slowdown is an artifact of the previous lower bound. We prove a constant-order characterization of $\Wfrak_{n,K}$. Specifically, for all $n\ge1$ and $K\ge2$, we show
\[
  \Wfrak_{n,K} = C_{\min\{n+1,K\}},
\]
where
\[
  C_q =
  \begin{cases}
    1, & q=1,\\
    \frac{q}{3q-4}, & q\ge2 \text{ and } q \text{ is even},\\
    \frac{q+1}{3q-1}, & q\ge2 \text{ and } q \text{ is odd}.
  \end{cases}
\]
The sequence $C_q$ decreases to $1/3$, and therefore $\Wfrak_{n,K}\in(1/3,1]$ for all $n\ge1$ and $K\ge2$. Our upper and lower bounds are \textbf{optimal} and \textbf{match exactly} at $C_{\min\{n+1,K\}}$.

Our proof separates the lower-bound argument into two complementary choices of vote vectors. The all-one vote vector controls the imbalance between the total weight on true classes and the total weight on false classes. Balanced or nearly balanced random vote vectors, depending on the parity of $K$, control the remaining regime by forcing each row to retain a constant fraction of its mass after the factorized reduction. Optimizing over the resulting one-dimensional parameter gives the constants $C_K$. Based on the lower bound $C_K$, we further find that the effective class number of the problem is actually $\min\{n+1,K\}$ and get a better lower bound $C_{\min\{n+1,K\}}$ by more fine-grained constructions. For the upper bound, we construct explicit worst-case weight and label matrices supported on $\min\{n,K\}$ rows and columns, and show that no vote vector can exceed $C_{\min\{n,K\}}$ on this construction. Moreover, we make another construction according to the effective class number conjecture and get a complementary upper bound. Combining the above upper and lower bounds lead to the exact characterization of $\Wfrak_{n,K}$.

As a consequence, Factorized AdaBoost.MH does not lose the boosting-type convergence rate of AdaBoost.MH except for a universal constant factor. Under the same empirical $\delta$-weak learning condition used in the standard analysis, the factorized algorithm needs only $O(\log(nK)/\delta^2)$ boosting rounds to achieve \(100\%\) training accuracy, which is of the same order as the original AdaBoost.MH algorithm. Thus the structured weak classifiers studied here preserve the essential weak-to-strong behavior of AdaBoost.MH while allowing the base learner to use a more coupled multi-class form.

The remainder of this article is structured as follows: Section~\ref{sec::problem-setup} introduces notations and setups of the task we study. Section~\ref{sec::main-theorem} presents our main theorems and the exact characterization of $\Wfrak_{n,K}$. Section~\ref{sec::proof-sketch-lower-bound} shows the proof sketch of the lower bounds. The conclusions are presented in the last section. Due to space limitation, related works are defered to Appendix~\ref{sec::related-works}. Detailed proofs can be found in Appendices~\ref{sec::proof-lower-bound}, \ref{sec::proof-upper-bound}, and \ref{sec::proof-others}.

\section{Problem Setup} \label{sec::problem-setup}

\paragraph{Notations.} For any positive integer $m$, we denote $[m]=\{1,\ldots,m\}$. We use bold lower-case letters for vectors and bold upper-case letters for matrices. For a matrix $\mA$, $\mA_{ij}$ denotes its $(i,j)$-th entry and $\mA_i$ denotes its $i$-th row. We use $\1[\cdot]$ for the indicator function. The symbol $\od$ denotes entrywise multiplication, and $\|\cdot\|_1$ denotes the vector $\ell_1$ norm.

We consider a multi-class classification problem where the input space is $\mathcal{X}$ and the output space is $\mathcal{Y} = [K]$, where $K \ge 2$. We consider $n$ training examples with $K$ classes. Let $\mY \in \{ \pm 1 \}^{n \times K}$ be the signed one-hot label matrix, where $\mY_{i,j} = +1$ if the $i$-th training example belongs to class $j$, and $\mY_{i,j} = -1$ otherwise. Formally, if the true label for the $i$-th training example is $y_i$, then for $j \in [K]$,
\begin{equation*}
\mY_{i,j} = \begin{cases}
+1 & \text{if } j = y_i \\
-1 & \text{if } j \neq y_i.
\end{cases}
\end{equation*}

AdaBoost.MH returns a vector-valued discriminant function $\vf: \mathcal{X} \to \mathbb{R}^{K}$ with a combined predictor $\tilde{\vf}: \mathcal{X} \to \PNOne^K$ where $\tilde{\vf}(x)_l = \sign(\vf(x)_l)$ for $l \in [K]$. Here we define
\begin{equation*}
  \sign(x) = \begin{cases}
  +1 & \text{if } x \ge 0 \\
  -1 & \text{if } x < 0.
  \end{cases}
\end{equation*}

Given the dataset $S = \{ (x_i, \vy_i) \}_{i=1}^n$ where $\vy_i$ is the $i$-th row of $\mY$, the aim is to learn a predictor $\tilde{\vf}$ such that $\tilde{\vf}(x_i) = \vy_i$ for all $i \in [n]$. To achieve this, AdaBoost.MH \citep{DBLP:journals/ml/SchapireS99} treats a $K$-class classification problem as $K$ binary classification problems and adapts the AdaBoost \citep{10.7551/mitpress/boosting-book} algorithm. Specifically, let $\mW \in [0,1]^{n \times K}$ be a normalized weight matrix where $\sum_{i=1}^n \sum_{j=1}^K \mW_{ij} = 1$, the aim is to minimize the weighted Hamming loss:
\begin{equation} \label{eq::Hamming-loss}
  R_\mathrm{H}(\tilde{\vf}, \mW) \coloneqq \sum_{i=1}^n \sum_{j=1}^K \mW_{ij} \1 \left[ \tilde{\vf}(x_i)_j \ne \mY_{ij}\right],
\end{equation}

where $\1 [\cdot]$ is the indicator function. It is easy to see that $\1 \left[ \tilde{\vf}(x_i)_j \ne \mY_{ij}\right] = \1 \left[ \vf(x_i)_j \cdot \mY_{ij} \le 0 \right] \le \exp \left( - \vf(x_i)_j \cdot \mY_{ij} \right)$, so the weighted Hamming loss is upper bounded by the weighted exponential loss

\begin{equation} \label{eq::exp-loss}
  R_\mathrm{EXP} (\vf, \mW) \coloneqq \sum_{i=1}^n \sum_{j=1}^K \mW_{ij} \exp\left( -\vf(x_i)_j \cdot \mY_{ij} \right).
\end{equation}

The pseudocode for AdaBoost.MH is shown in Algorithm \ref{alg:adaboost.mh}. According to \citep[Chapter 7]{10.7551/mitpress/boosting-book}, AdaBoost.MH directly minimizes the weighted exponential loss and outputs a combined predictor $\vf^{(T)}(\cdot)$, where $\vf^{(T)}(x) = \sum_{t=1}^T \vh^{(T)}(x)$ is a sum of $T$ base predictors $\vh^{(t)}: \mathcal{X} \to \mathbb{R}^K$ returned by a base learner $A(\mX, \mY, \mW^{(t)})$, where $\mX_i = x_i$ and $\mW^{(t)}$ is the weight matrix for the $t$-th iteration. According to the calculation in \citep[Proof of Theorem 3.1]{10.7551/mitpress/boosting-book}, given $\mX$ and $\mY$, it is easy to know that:
\begin{equation} \label{eq::exp-loss-prod-Z}
  R_\mathrm{EXP}(\vf^{(T)}, \mW) = \prod_{t=1}^T \underbrace{\left( \sum_{i=1}^n \sum_{j=1}^K \mW_{ij}^{(t)} \exp\left( - \vh^{(t)} (x_i)_j \cdot \mY_{ij} \right) \right)}_{\coloneqq Z(\vh^{(t)}, \mW^{(t)})} = \prod_{t=1}^T Z(\vh^{(t)}, \mW^{(t)}).
\end{equation}

\begin{algorithm}[!ht]
\caption{AdaBoost.MH} \label{alg:adaboost.mh}
\SetKwInOut{Input}{Input}
\SetKwInOut{Output}{Output}

\Input{Training data $\mX, \mY$; initialization weight matrix $\mW$; base learner $A(\cdot,\cdot,\cdot)$; number of iterations $T$;}
$\mW^{(1)} = \mW$\;

\For{$t\leftarrow 1\  \mathrm{to}\  T$}{
$\vh^{(t)}(\cdot) \leftarrow A(\mX,\mY,\mW^{(t)})$\;
\For{$i \leftarrow 1 \  \mathrm{to}\ n$}{
\For{$j \leftarrow 1 \  \mathrm{to}\ K$}{
$\mW_{ij}^{(t+1)} \leftarrow \mW_{ij}^{(t)} \cdot 
\frac{
\exp{\left(-\vh^{(t)}(x_i)_j \cdot \mY_{ij} \right)}
}{
\sum_{i^\prime = 1}^n \sum_{j^\prime = 1}^K 
\mW_{i^\prime j^\prime}^{(t)}
\exp\left( -\vh^{(t)}(x_{i^\prime})_{j^\prime} \cdot \mY_{i^\prime j^\prime} \right)
}$\;
}
}
}
\Output{$\vf^{(T)} (\cdot) = \sum_{t=1}^T \vh^{(t)}(\cdot)$;}
\end{algorithm}

Based on \Cref{eq::exp-loss-prod-Z}, we know that to minimize $R_\mathrm{EXP}(\vf^{(T)}, \mW)$, the base learner $A(\mX,\mY,\mW^{(t)})$ should find a $\vh^{(t)}$ to minimize $Z(\vh^{(t)}, \mW^{(t)})$. 

In the original AdaBoost.MH algorithm, \citet{DBLP:journals/ml/SchapireS99} consider modeling $\vh$ as the product of a scalar $\alpha \in \mathbb{R}$ and a vectorized mapping $\vvarphi: \mathcal{X} \to \PNOne^K$, i.e., $\vh(x)=\alpha \vvarphi(x)$. It can be shown \citep{DBLP:journals/ml/SchapireS99,DBLP:conf/nips/ZouZXL24} that under a standard empirically $\delta$-weak learning condition (Assumption~\ref{ass::empirical-delta-weak-learning-condition}), AdaBoost can reach the following upper bound for the weighted exponential loss:
\begin{equation*}
  R_\mathrm{EXP}(\vf^{(T)}, \mW) \le \exp \left( -\frac{\delta^2}{2} T \right).
\end{equation*}

\begin{assumption}[empirically $\delta$-weak learing condition, Definition 2.1 of \citep{DBLP:conf/nips/ZouZXL24}] \label{ass::empirical-delta-weak-learning-condition}
For a given binary dataset $\{ (x_1, y_1), \dots, (x_m, y_m) \}$ where $y_i \in \PNOne$, we assume that for $\delta > 0$, for any distribution $\vw \in \Delta^{m-1} \coloneqq \left\{ \bm{\lambda} \in \mathbb{R}^m \Big| \bm{\lambda}_i \ge 0\ \forall i \in [m], \sum_{i=1}^m \bm{\lambda}_i = 1 \right\}$ over $[m]$, we can always find a binary classifier $\varphi: \mathcal{X} \to \PNOne$ such that:
\begin{equation*}
  \gamma = \sum_{i=1}^m \vw_i \cdot y_i \cdot \varphi(x_i) \ge \delta.
\end{equation*}
\end{assumption}

AdaBoost.MH reduces the multi-class problem into $K$ binary one-against-all classifications, however, \citet{DBLP:journals/corr/Kegl13} avoids such a reduction by factorizing the vector-valued classifier $\vh$ into an input-independent vector of length $K$ and a label-independent scalar classifier. Formally, \citet{DBLP:journals/corr/Kegl13} sets
\begin{equation*}
  \vh(x) = \alpha \vv \varphi(x),
\end{equation*}
where $\alpha \in \mathbb{R}_+$ is a positive real-valued base coefficient, $\vv \in \PNOne^K$ is an input-independent vote vector of length $K$, and $\varphi: \mathcal{X} \to \PNOne$ is a label-independent binary classifier.

Now we derive the corresponding exponential-loss upper bound. We focus on one boosting iteration and omit the superscript $t$ for readability. For a fixed triple $(\alpha,\vv,\varphi)$, define the factorized edge
\begin{equation} \label{eq::factorized-edge}
  \gamma(\vv,\varphi,\mW)
  \coloneqq
  \sum_{i=1}^n \sum_{j=1}^K
  \mW_{ij} \cdot \mY_{ij} \cdot \vv_j \cdot \varphi(x_i).
\end{equation}
This quantity is the weighted agreement between the factorized base classifier and the signed labels: a pair $(i,j)$ contributes positively when $\vv_j\varphi(x_i)$ has the same sign as $\mY_{ij}$, and negatively otherwise. According to the similar derivation in \citet{DBLP:conf/nips/ZouZXL24}, since $\vv_j\varphi(x_i)\mY_{ij} \in \PNOne$ and $\sum_{i,j}\mW_{ij}=1$, we have
\begin{align}
  Z(\vh,\mW)
  &=
  \sum_{i=1}^n \sum_{j=1}^K
  \mW_{ij}\exp\left(-\alpha \vv_j\varphi(x_i)\mY_{ij}\right) \notag \\
  &=
  \frac{1+\gamma(\vv,\varphi,\mW)}{2}e^{-\alpha}
  +
  \frac{1-\gamma(\vv,\varphi,\mW)}{2}e^{\alpha}.
  \label{eq::factorized-Z-edge}
\end{align}
Therefore, after fixing $\vv$ and $\varphi$, minimizing over $\alpha$ gives
\begin{equation*}
  \alpha
  =
  \frac{1}{2}\ln\left(
  \frac{1+\gamma(\vv,\varphi,\mW)}
       {1-\gamma(\vv,\varphi,\mW)}
  \right),
  \qquad
  Z(\vh,\mW)=\sqrt{1-\gamma(\vv,\varphi,\mW)^2}.
\end{equation*}
Thus the factorized base learner should find $\vv$ and $\varphi$ with a large positive edge. If for every iteration there is a uniform lower bound $\gamma(\vv^{(t)},\varphi^{(t)},\mW^{(t)})\ge \delta$, then by \Cref{eq::exp-loss-prod-Z},
\begin{equation} \label{eq::factorized-exp-loss-delta}
  R_\mathrm{EXP}(\vf^{(T)},\mW)
  \le
  \prod_{t=1}^T \sqrt{1-\delta^2}
  \le
  \exp\left(-\frac{\delta^2}{2}T\right).
\end{equation}

The remaining difficulty is to justify such a positive lower bound on the factorized edge. In the unfactorized case, one may choose a different binary classifier for each label $j$. In the factorized case, the same scalar classifier $\varphi$ must serve all labels, so Assumption~\ref{ass::empirical-delta-weak-learning-condition} cannot be applied to each column of $\mY$ independently. Following the reduction described in \citet{DBLP:conf/colt/Kegl14,DBLP:conf/nips/ZouZXL24}, fix a vote vector $\vv$ and rewrite \Cref{eq::factorized-edge} by collecting all label contributions that belong to the same training example, we define
\begin{equation*}
  w_i^+
  \coloneqq
  \sum_{j=1}^K \mW_{ij}\1[\vv_j\mY_{ij}=+1],
  \qquad
  w_i^-
  \coloneqq
  \sum_{j=1}^K \mW_{ij}\1[\vv_j\mY_{ij}=-1].
\end{equation*}
Then
\begin{equation}\label{eq::factorized-edge-pseudo}
  \gamma(\vv,\varphi,\mW)
  =
  \sum_{i=1}^n
  \varphi(x_i)
  \sum_{j=1}^K
  \mW_{ij}\vv_j\mY_{ij} =
  \sum_{i=1}^n
  \varphi(x_i)(w_i^+-w_i^-).
\end{equation}

This turns the multi-class weighted problem into an ordinary binary weighted problem on the original examples. Let
\begin{equation*}
  y_i^\prime \coloneqq \sign(w_i^+-w_i^-),
  \qquad
  w_i^\prime \coloneqq |w_i^+-w_i^-|,
  \qquad
  \wSig \coloneqq \sum_{i=1}^n w_i^\prime.
\end{equation*}
Using these pseudo-labels and pseudo-weights, \Cref{eq::factorized-edge-pseudo} becomes
\begin{equation} \label{eq::factorized-edge-pseudo-weight}
  \gamma(\vv,\varphi,\mW)
  =
  \sum_{i=1}^n w_i^\prime y_i^\prime \varphi(x_i).
\end{equation}
However, the vector $(w_1^\prime,\ldots,w_n^\prime)$ is not necessarily a distribution because $\wSig$ can be smaller than $1$. If $\wSig>0$, then after normalization, Assumption~\ref{ass::empirical-delta-weak-learning-condition} gives a binary classifier $\varphi$ such that
\begin{equation*}
  \sum_{i=1}^n
  \frac{w_i^\prime}{\wSig}
  y_i^\prime \varphi(x_i)
  \ge \delta.
\end{equation*}
Combining this with \Cref{eq::factorized-edge-pseudo-weight} yields
\begin{equation} \label{eq::factorized-edge-wSigma}
  \gamma(\vv,\varphi,\mW)
  \ge
  \delta \wSig,
\end{equation}
which leads to an upper bound for the weighted exponential loss:
\begin{equation*}
  R_\mathrm{EXP}(\vf^{(T)}, \mW) \le \exp \left( -\frac{\delta^2 \left(\wSig\right)^2}{2} T \right).
\end{equation*}

Consequently, to obtain the exponential-loss bound in \Cref{eq::factorized-exp-loss-delta} for Factorized AdaBoost.MH, it remains to choose a vote vector $\vv$ for which $\wSig$ has a nontrivial lower bound. The next step is therefore to prove a lower bound on $\wSig$ that holds uniformly over the current weight matrix $\mW$ and label matrix $\mY$. Regarding this step, \citet{DBLP:conf/colt/Kegl14} raises an open problem:

\begin{quote}
Does there exist a setup $(\mX,\mW,\mY)$ such that all $2^K$ possible vote vectors $\vv\in\PNOne^K$ lead to arbitrarily small, or even zero, $\wSig$? Or can one find a constant (independent of $n$) lower bound $\omega>0$, independent of the sample size $n$, such that for at least one vote vector $\vv$ and classifier $\varphi$, $\wSig\ge \omega$ holds?
\end{quote}

Since $\wSig$ is fully decided by $\mW,\mY,\vv$, we write $\wSig(\mW,\mY,\vv)$ to explicitly show such a relationship. \citet{DBLP:conf/nips/ZouZXL24} reformulate the above open problem as finding a lower bound for
\begin{equation} \label{eq::wSigma-minimax}
  \begin{aligned}
    \mathfrak{W}_{n,K} &\coloneqq \min_{\mW\in\Wcal_{n,K},\ \mY\in\Ycal_{n,K}}
  \max_{\vv\in \Vcal_K} \wSig(\mW, \mY, \vv) \\
  &= \min_{\mW\in\Wcal_{n,K},\ \mY\in\Ycal_{n,K}}
  \max_{\vv\in \Vcal_K} \left\|(\mW\odot\mY)\vv\right\|_1,
  \end{aligned}
\end{equation}
where
\begin{equation*}
  \Wcal_{n,K}
  \coloneqq
  \left\{
  \mW\in\mathbb{R}^{n\times K}
  \,\middle|\,
  \mW_{ij}\ge 0,\ 
  \sum_{i=1}^n\sum_{j=1}^K \mW_{ij}=1
  \right\}
\end{equation*}
is the set of all normalized weight matrices,
\begin{equation*}
  \Ycal_{n,K}
  \coloneqq
  \left\{
  \mY\in\PNOne^{n\times K}
  \,\middle|\,
  \text{each row of $\mY$ is a signed one-hot label vector}
  \right\}
\end{equation*}
is the set of all possible signed one-hot label matrices, and $\Vcal_K\coloneqq\PNOne^K$ is the set of all vote vectors. The notation $\odot$ denotes entrywise multiplication and $\|\cdot\|_1$ denotes the vector $\ell_1$ norm.

\citet[Theorems 3.3-3.4]{DBLP:conf/nips/ZouZXL24} solve the above open problem by proving that $\mathfrak{W}_{n,K} \ge \max \left\{ \frac{1}{n}, \frac{1}{\sqrt{2K}} \right\}$. However, it is unclear whether the lower bound shown by \citet{DBLP:conf/nips/ZouZXL24} is tight. In this paper, we show that the lower bound in \citep{DBLP:conf/nips/ZouZXL24} is not tight, and provide an \textbf{exact characterization} of $\mathfrak{W}_{n,K}$.

\section{Main Theorems} \label{sec::main-theorem}

In this section, we present our main result, the exact characterization of $\mathfrak{W}_{n,K}$. We first take a step to provide a better (compared to those in \citep{DBLP:conf/nips/ZouZXL24}) lower bound for $\mathfrak{W}_{n,K}$, and then show the exact upper and lower bounds for $\mathfrak{W}_{n,K}$ based on our initial bounds.

\begin{theorem}[Upper and Lower Bounds for $\mathfrak{W}_{n,K}$] \label{thm::main-theorem}
  For all integers $n \ge 1$ and $K \ge 2$, we have:
  \begin{equation} \label{ieq::upper-lower-bound}
    \max \left\{ \frac{1}{n}, C_K \right\} \le \mathfrak{W}_{n,K} \le C_{\min\{n,K\}},
  \end{equation}
  where
  \begin{equation} \label{eq::Cq}
    C_q = \begin{cases}
    1 & q=1 \\
    \frac{q}{3q-4} & q\ge2 \text{ and } q \text{ is even}, \\
    \frac{q+1}{3q-1} & q\ge2 \text{ and } q \text{ is odd}.
    \end{cases}
  \end{equation}
\end{theorem}

\begin{proposition}[Properties of $C_q$] \label{prop::Cq-properties}
  The sequence $C_q$ defined in \Cref{eq::Cq} has the following properties:
  \begin{enumerate}
    \item $C_q$ is non-increasing with respect to $q$.
    \item If $q$ is odd, then $C_{q+1} = C_q$.
    \item $C_q$ tends to $\frac{1}{3}$ as $q$ tends to infinity.
  \end{enumerate}
\end{proposition}

\Cref{thm::main-theorem} provides a lower bound $\Wfrak_{n,K} \ge C_K$. When $n \gg K$, then the $n$ examples are likely to contain all $K$ labels. However, when $n < K$, the $n$ examples can not use up all $K$ labels. So we conjecture that there might be \textbf{effective class number} that really captures the intrinsic hardness caused by the labels. Based on this conjecture, we narrow down the effect of label numbers to $n+1$ (rather than $K$ shown in \Cref{thm::main-theorem}) when $n + 1 < K$ and get the following refined lower bound.

\begin{theorem}[Refined Lower Bound for $\mathfrak{W}_{n,K}$] \label{thm::refine-lower-bound}
  For all integers $n \ge 1$ and $K \ge 2$, we have:
  \begin{equation} \label{ieq::refined-lower-bound}
    \mathfrak{W}_{n,K} \ge C_{\min\{n+1,K\}}.
  \end{equation}
\end{theorem}

Based on the \textbf{effective class number} conjecture, we further construct a new upper bound as follows, which makes up the exact upper bound together with the upper bound in \Cref{thm::main-theorem}.

\begin{theorem}[Another Upper Bound for $\mathfrak{W}_{n,K}$] \label{thm::another-upper-bound}
  For integers $n \ge 1$ and $K \ge 2$, when $n$ is even and $n+1 \le K$, we have:
  \begin{equation} \label{ieq::another-lower-bound}
    \mathfrak{W}_{n,K} \le C_{n+1}.
  \end{equation}
\end{theorem}

Combining \Cref{thm::main-theorem,thm::refine-lower-bound,thm::another-upper-bound}, we get \textbf{optimal} upper and lower bounds, leading to the following \textbf{exact characterization} of $\Wfrak_{n,K}$.

\begin{theorem}[Exact Value of $\mathfrak{W}_{n,K}$] \label{thm::exact-value}
  For integers $n \ge 1$ and $K \ge 2$, we have:
  \begin{equation} \label{eq::exact-value}
    \mathfrak{W}_{n,K} = C_{\min\{n+1,K\}}.
  \end{equation}
\end{theorem}

\begin{remark}
  Theorem \ref{thm::exact-value} tells us that,
  \begin{equation*}
    \mathfrak{W}_{n,K} = C_{\min\{n+1,K\}}.
  \end{equation*}
  By Proposition \ref{prop::Cq-properties}, we know that $C_q \le 1$ and decreases to $\frac{1}{3}$ as $q$ tends to infinity, it is easy to see that:
  \begin{equation*}
    \frac{1}{3} < C_K \le \mathfrak{W}_{n,K} = C_{\min\{n+1,K\}} \le 1,
  \end{equation*}
  which means that $\Wfrak_{n,K} = \Theta(1)$ and is a great improvement upon the results in \citep{DBLP:conf/nips/ZouZXL24}. When $n,K \to \infty$, we have $\Wfrak_{n,K} \to \frac{1}{3}$.
\end{remark}

Let's now compare the results in \citep{DBLP:conf/nips/ZouZXL24} and those in this paper. The bounds in \citep{DBLP:conf/nips/ZouZXL24} are
\begin{equation} \label{ieq::bounds-previous}
  \max\left\{ \frac{1}{n}, \frac{1}{\sqrt{2K}} \right\} \le \Wfrak_{n,K} \le 1.
\end{equation}
\citet{DBLP:conf/nips/ZouZXL24} claim that $\Wfrak_{n,K} > 0$ is guaranteed when $n$ and $K$ do not tend to infinity simultaneously. However, this paper shows a much more stronger result: $\Wfrak_{n,K} = C_{\min\{n+1,K\}} \in \left(\frac{1}{3},1\right]$ whenever $n \ge 1$ and $K \ge 2$.

If we set the original input matrix in Algorithm \ref{alg:adaboost.mh} as $\mW_{ij} = \frac{1}{nK}$ for all $i \in [n], j \in [K]$, then by making the upper bound of the exponential loss less than $\frac{1}{nK}$, we can get that
\begin{equation*}
  T > \frac{2 \log (nK)}{\delta^2}.
\end{equation*}
This means that if we run the original AdaBoost.MH algorithm for $T_\mathrm{ori} \coloneqq \left\lceil\frac{2 \log (nK)}{\delta^2}\right\rceil + 1$ steps, then can make sure the combined predictor correctly classify all training data.

Similarly, in the factorized case, solving the equation $\exp \left( -\frac{\delta^2 \left(\wSig\right)^2}{2} T \right) < \frac{1}{nK}$ through the lower bound given by \citet{DBLP:conf/nips/ZouZXL24}, we can get
\begin{equation} \label{eq::T-facotrized-case}
  T_\mathrm{fac} \coloneqq \left\lceil \frac{ 2 \log (nK)}{\delta^2 \left( \wSig \right)^2} \right\rceil + 1.
\end{equation}
\Cref{eq::T-facotrized-case} shows that if we run Factorized AdaBoost.MH for $T_\mathrm{fac}$ steps, then the combined (factorized) predictor can correctly classify all training data. Combining the lower bound in \citep{DBLP:conf/nips/ZouZXL24}, we show that
\begin{equation} \label{eq::T-facotrized-case-ZouZXL}
  T_\mathrm{fac,1} \coloneqq \left\lceil \frac{ 2 \log (nK) \cdot \min \{n^2, 2K \}}{\delta^2} \right\rceil + 1
\end{equation}
steps are sufficient to make sure the convergence of the Factorized AdaBoost.MH algorithm. Compare $T_\mathrm{ori}$ and $T_\mathrm{fac,1}$, we find that according to the results in \citep{DBLP:conf/nips/ZouZXL24}, the iteration number of Factorized AdaBoost.MH is about $\min \{n^2, 2K \}$ times more than that of the original AdaBoost.MH algorithm. When $K$ and $n$ are large, the $\min \{n^2, 2K \}$ term shows that Factorized AdaBoost.MH may need a lot more steps to converge.

However, our results show that the bounds in \citep{DBLP:conf/nips/ZouZXL24} are not tight, and the additional $\min \{n^2, 2K \}$ term is not needed. Specifically, our tight lower bound for $\Wfrak_{n,K}$ shows that
\begin{equation} \label{eq::T-facotrized-case-this-paper}
  T_\mathrm{fac,2} \coloneqq \left\lceil\frac{18 \log (nK)}{\delta^2}\right\rceil + 1
\end{equation}
steps are enough to make sure the Factorized AdaBoost.MH algorithm converges. Compared with $T_\mathrm{ori}$ and \Cref{eq::T-facotrized-case-ZouZXL}, our \Cref{eq::T-facotrized-case-this-paper} tells us that for the AdaBoost.MH algorithm, \textbf{we can use factorized classifiers without loss of the convergence rate}.

\section{Proof Sketch of the Lower Bound} \label{sec::proof-sketch-lower-bound}
The proof of the lower bound consists of two parts. The first part is the proof of the lower bound in \Cref{thm::main-theorem}, the second part uses a more fine-graind construction based on the effective class number conjecture. Due to space limitation, in this section, we provide the proof sketch for the lower bounds in \Cref{thm::main-theorem}. For complete proof details, please refer to Appendix~\ref{sec::proof-lower-bound}.

According to Theorem 3.3 in \citep{DBLP:conf/nips/ZouZXL24}, we have $\mathfrak{W}_{n,K} \ge \frac{1}{n}$. So it remains to prove $\mathfrak{W}_{n,K} \ge C_K$. Before presenting the main proof steps for \Cref{thm::main-theorem}, we first introduce some definitions for convenience.

Fix any $\mW \in \Wcal_{n,K}$ and $\mY \in \Ycal_{n,K}$, for any $i \in [n]$, let $c_i \in [K]$ be the column unique index such that $\mY_{i c_i} = +1$. We define the row mass for the $i$-th row as
\begin{equation*}
  s_i = \sum_{j=1}^K \mW_{ij}.
\end{equation*}

Rows with $s_i=0$ do not affect $\wSig$, so we consider the rows with $s_i > 0$. For rows with $s_i > 0$, we define the true-class weight fraction as
\begin{equation*}
  a_i = \frac{\mW_{i c_i}}{s_i}.
\end{equation*}

It is easy to see that $0 \le a_i \le 1$ and $\sum_{i=1}^n s_i = 1$. Then the total weight for the true classes is
\begin{equation*}
  \rho = \sum_{i=1}^n s_i a_i = \sum_{i=1}^n \mW_{i c_i}.
\end{equation*}

Since $\sum_{i=1}^n s_i = 1$ and $0 \le a_i \le 1$ for all $i \in [n]$, we know that $0\le \rho \le 1$. Fix $\mW,\mY$, we define
\begin{equation*}
  Z_i(\vv) = \sum_{j=1}^K \mW_{ij} \mY_{ij} \vv_j,
\end{equation*}
then, our goal is to make sure that make sure there exists a vector $\vv$ such that $\wSig(\mW,\mY,\vv) = \sum_{i=1}^n |Z_i(\vv)|$ is away from zero. Our main strategy is to construct two different types of vector sets so that $Z_i(\vv)$ cannot simultaneously be small for the two sets.

The first type of vector we consider is the all-one vector $\vv^1 = (+1, \dots, +1)$. For $\vv^1$, $|Z_i(\vv^1)| = \left| \mW_{i c_i} - \sum_{j \ne c_i} \mW_{ij} \right|$, i.e., the difference between the true label mass and the false label mass. To make $|Z_i(\vv^1)|$ small, the mass of the true label and false labels should be close, making $\rho \approx \frac{1}{2}$. Our following lemma formulates such an intuition.

\begin{lemma}[All-one Vector] \label{lem::all-one-lower-bound} 
  Fix integers $n \ge 1$ and $K \ge 2$, for all $\mW \in \Wcal_{n,K}$ and all $\mY \in \Ycal_{n,K}$, we have that:
  \begin{equation}\label{ieq::all-one-lower-bound}
    \underset{\vv \in \Vcal_K}{\max}\wSig(\mW,\mY,\vv) \ge |2\rho -1|.
  \end{equation}
\end{lemma}

Other than the all-one vector, we need to construct another type of vectors which can make $\sum_{i=1}^n |Z_i(\vv)|$ away from zero when the all-one vector fails to do so (i.e., when $\rho \approx \frac{1}{2}$). Since $|Z_i(\vv)| = \left| \mW_{i c_i} \vv_{c_i} - \sum_{j \ne c_i} \mW_{ij} \vv_j \right|$, conditioned on the fact that $\mW_{i c_i} \approx \sum_{j \ne c_i} \mW_{ij}$ (i.e., the fact that all-one vector fails), a natural idea to make this term away from zero is to set balanced vectors $\vv$ so that the positive and negative contributions in $\sum_{j \ne c_i} \mW_{ij} \vv_j$ cancel each other as much as possible. The following two lemmas formulate the above intuition and construct a balanced set and a nearly balanced set based on the parity of $K$.

\begin{lemma}[Balanced Vector Set for Even $K$] \label{lem::balance-even-lower-bound}
  Fix integers $n \ge 1$ and $K \ge 2$, when $K$ is \textbf{even}, then for all $\mW \in \Wcal_{n,K}$ and all $\mY \in \Ycal_{n,K}$, we have that:
  \begin{equation} \label{ieq::balance-even-lower-bound}
    \underset{\vv \in \Vcal_K}{\max}\wSig(\mW,\mY,\vv) \ge \rho + \frac{1-\rho}{K-1}.
  \end{equation}
\end{lemma}

\begin{lemma}[Nearly Balanced Vector Set for Odd $K$] \label{lem::near-balance-even-lower-bound}
  Fix integers $n \ge 1$ and $K \ge 3$, when $K$ is \textbf{odd}, then for all $\mW \in \Wcal_{n,K}$ and all $\mY \in \Ycal_{n,K}$, we have that:
  \begin{equation} \label{ieq::near-balance-even-lower-bound}
    \underset{\vv \in \Vcal_K}{\max}\wSig(\mW,\mY,\vv) \ge \rho + \frac{1-\rho}{K}.
  \end{equation}
\end{lemma}

Let's take the even case as an example. Combining Lemma \ref{lem::all-one-lower-bound} and Lemma \ref{lem::balance-even-lower-bound}, we know that when $\rho$ is away from $\frac{1}{2}$, i.e., $|2\rho -1| > 0$, then we can get $\underset{\vv \in \Vcal_K}{\max}\wSig(\mW,\mY,\vv) \ge |2\rho -1|$ according to \Cref{ieq::all-one-lower-bound}; when $\rho$ is close to $\frac{1}{2}$, then according to \Cref{ieq::balance-even-lower-bound} we know that $\underset{\vv \in \Vcal_K}{\max}\wSig(\mW,\mY,\vv) \ge \rho + \frac{1-\rho}{K-1} \approx \frac{1}{2} + \frac{1}{2(K-1)} > \frac{1}{2}$. So our strategy can make sure at least one of the two kinds of vectors make $\wSig(\mW,\mY,\vv)$ away from zero. Putting Lemmas~\ref{lem::all-one-lower-bound} and \ref{lem::balance-even-lower-bound} together, we know that when $K$ is even and $K \ge 2$,
\begin{equation*}
  \Wfrak_{n,K} \ge \underset{\mW \in \Wcal_{n,K}, \mY \in \Ycal_{n,K}}{\min} \max\left\{ |2\rho-1|, \rho + \frac{1-\rho}{K-1} \right\} \ge \underset{\rho \in [0,1]}{\min} \max\left\{ |2\rho-1|, \rho + \frac{1-\rho}{K-1} \right\}.
\end{equation*}
Minimizing $\underset{\rho \in [0,1]}{\min} \max\left\{ |2\rho-1|, \rho + \frac{1-\rho}{K-1} \right\}$ provides our lower bound for even $K$. For complete and formal arguments, please refer to Appendix \ref{subsec::proof-lower-bound-raw}.

\section{Conclusion}

In this paper, we revisited the convergence analysis of Factorized
AdaBoost.MH through the combinatorial quantity that controls the induced
binary weight mass after the factorized reduction. Previous guarantees showed that this quantity is always positive, but still allowed a dimension-dependent loss in the number of boosting rounds. We proved a sharper characterization: for all \(n\ge1\) and \(K\ge2\), the minimax value \(\Wfrak_{n,K}\) is bounded below by \(\max\{1/n,C_K\}\) and above by \(C_{\min\{n,K\}}\), where the constants \(C_q\) are explicit and converge to \(1/3\). Thus the relevant factorized edge is uniformly of constant order. As a consequence, Factorized AdaBoost.MH retains the boosting-type convergence rate of AdaBoost.MH up to a universal constant factor under the same empirical weak learning condition. This closes the remaining gap in the convergence analysis of factorized multi-class weak classifiers and supports their use as a structured alternative to independent one-against-all weak rules.

\bibliography{iclr2026_conference}
\bibliographystyle{iclr2026_conference}

\appendix

\section{Related Works} \label{sec::related-works}

\paragraph{Binary boosting and weak-to-strong learning.}
Boosting originated as a constructive answer to the question of how a weak learning procedure can be converted into a strong one. AdaBoost is the most prominent representative of this idea: it repeatedly changes the distribution over training examples, calls a weak learner, and combines the returned rules with data-dependent coefficients. Its classical analysis connects three important viewpoints: adaptive reweighting, minimization of exponential loss, and geometric decay of the training error under a positive-edge condition
\citep{10.7551/mitpress/boosting-book}. This combination of a simple algorithm and a sharp theoretical guarantee is one reason boosting remains a basic tool for studying supervised learning algorithms \citep{DBLP:journals/jmlr/Hanneke16,DBLP:conf/colt/MontasserHS19}.

\paragraph{Multi-class boosting and AdaBoost.MH.}
The multi-class setting introduces an additional structural question: should one reduce the task to several binary problems, or should the weak learner operate directly on class-vector predictions? AdaBoost.MH follows the first route in a particularly influential way. It represents each label by a signed vector, places weights on example--label pairs, and uses an exponential surrogate for the weighted Hamming loss \citep{DBLP:journals/ml/SchapireS99}. This framework keeps the binary AdaBoost proof mechanism largely intact because a vector-valued weak rule may be regarded as a collection of class-coordinate binary rules. As a result, AdaBoost.MH serves as a natural bridge from binary boosting to multi-class boosting.

\paragraph{Structured and factorized weak classifiers.}
Although the one-against-all view used in the original AdaBoost.MH algorithm is mathematically convenient, it can ignore the dependencies among class labels. \citet{DBLP:journals/corr/Kegl13} proposed to use more structured multi-class weak learners and introduced factorized rules of the form \(\vh(x)=\alpha\vv\varphi(x)\). In this representation, the scalar
classifier \(\varphi\) determines an input-dependent split, while the vote vector \(\vv\) assigns signs to the classes. This gives a compact way to share one weak decision across all labels and is closely connected to the multi-class Hamming-tree viewpoint \citep{DBLP:journals/corr/Kegl13}. The benefit of this structure is that it provides a more coupled and potentially more efficient base learner; the cost is that the standard AdaBoost.MH convergence argument cannot be applied to the factorized weak classifiers.

\paragraph{Convergence of Factorized AdaBoost.MH.}
The missing convergence step for Factorized AdaBoost.MH was isolated by \citet{DBLP:conf/colt/Kegl14}. After fixing a vote vector, the weighted example--label problem collapses to a binary problem on the original examples, but only with total induced mass \(\|(\mW\odot\mY)\vv\|_1\). Therefore the weak learning condition yields a useful factorized edge only if some vote vector preserves enough mass for every possible current weight matrix and label matrix. \citet{DBLP:conf/nips/ZouZXL24}
proved that such a vote vector always exists and obtained the lower bound \(\Wfrak_{n,K}\ge\max\{1/n,1/\sqrt{2K}\}\), thereby giving a boosting-type convergence theorem for Factorized AdaBoost.MH. However, the above lower bound still leaves a possible slowdown depending on the number of samples or classes. Our contribution is to sharpen this combinatorial bottleneck: we prove that the relevant minimax quantity is bounded below by a universal constant, with the explicit constants \(C_{\min \{ n+1, K\}}\) and matching upper bounds described in the main theorem.

\section{Proof of the Lower Bounds} \label{sec::proof-lower-bound}

\subsection{Proof of the Lower Bound in \Cref{thm::main-theorem}} \label{subsec::proof-lower-bound-raw}
In this section, we provide the proof for the lower bounds in \Cref{thm::main-theorem}. According to Theorem 3.3 in \citep{DBLP:conf/nips/ZouZXL24}, we have $\mathfrak{W}_{n,K} \ge \frac{1}{n}$. So it remains to prove $\mathfrak{W}_{n,K} \ge C_K$.

Fix any $\mW \in \Wcal_{n,K}$ and $\mY \in \Ycal_{n,K}$, for any $i \in [n]$, let $c_i \in [K]$ be the column unique index such that $\mY_{i c_i} = +1$. We define the row mass for the $i$-th row as
\begin{equation*}
  s_i = \sum_{j=1}^K \mW_{ij}.
\end{equation*}

Rows with $s_i=0$ do not affect $\wSig$, so we consider the rows with $s_i > 0$. For rows with $s_i > 0$, we define the true-class weight fraction as
\begin{equation*}
  a_i = \frac{\mW_{i c_i}}{s_i}.
\end{equation*}

It is easy to see that
\begin{equation*}
  0 \le a_i \le 1,  \quad \sum_{i=1}^n s_i = 1.
\end{equation*}

Then the total weight for the true classes is
\begin{equation*}
  \rho = \sum_{i=1}^n s_i a_i = \sum_{i=1}^n \mW_{i c_i}.
\end{equation*}

Since $\sum_{i=1}^n s_i = 1$ and $0 \le a_i \le 1$ for all $i \in [n]$, we know that $0\le \rho \le 1$. Next, given fixed $\mW,\mY$, we provide three lower bounds for $\underset{\vv \in \Vcal_K}{\max}\wSig(\mW,\mY,\vv)$.

\begin{customlemma}{\ref{lem::all-one-lower-bound}}
  Fix integers $n \ge 1$ and $K \ge 2$, for all $\mW \in \Wcal_{n,K}$ and all $\mY \in \Ycal_{n,K}$, we have that:
  \begin{equation*}
    \underset{\vv \in \Vcal_K}{\max}\wSig(\mW,\mY,\vv) \ge |2\rho -1|.
  \end{equation*}
\end{customlemma}

\begin{proof}[Proof of Lemma \ref{lem::all-one-lower-bound}]
  We choose all-one vector to show this lower bound. Take $\vv^1 = (+1, \dots, +1) \in \Vcal_K$, we have that, for the $i$-th row,
  \begin{equation*}
    \left| \sum_{j=1}^K \mW_{ij} \mY_{ij} \vv_j^1 \right| = \left| \sum_{j=1}^K \mW_{ij} \mY_{ij} \right| = \left| \mW_{i c_i} - \sum_{j\ne c_i}^K \mW_{ij} \right| \overset{(a)}{=} \left| s_i a_i - s_i(1-a_i) \right| = s_i |2a_i -1|,
  \end{equation*}
  where step $(a)$ follows from the facts that $\mW_{i c_i} = s_i a_i$ and $\sum_{j\ne c_i}^K \mW_{ij} = s_i (1 - a_i)$. Then we have:
  \begin{equation*}
    \underset{\vv \in \Vcal_K}{\max}\wSig(\mW,\mY,\vv) \ge \wSig(\mW,\mY,\vv^1) = \sum_{i=1}^n s_i |2a_i -1| \ge \left| \sum_{i=1}^n s_i (2a_i -1) \right| \overset{(a)}{=} |2 \rho -1|,
  \end{equation*}
  where step $(a)$ comes from the definition of $\rho$ and the fact that $\sum_{i=1}^n s_i =1$.
\end{proof}

\begin{customlemma}{\ref{lem::balance-even-lower-bound}}
  Fix integers $n \ge 1$ and $K \ge 2$, when $K$ is \textbf{even}, then for all $\mW \in \Wcal_{n,K}$ and all $\mY \in \Ycal_{n,K}$, we have that:
  \begin{equation*}
    \underset{\vv \in \Vcal_K}{\max}\wSig(\mW,\mY,\vv) \ge \rho + \frac{1-\rho}{K-1}.
  \end{equation*}
\end{customlemma}

\begin{proof}[Proof of Lemma \ref{lem::balance-even-lower-bound}]
  Since $K \ge 2$ is even, we can define a set of balanced vectors where the numbers of $-1$ and $+1$ are the same:
  \begin{equation*}
    \Bcal_K \coloneqq \left\{ \vv \in \PNOne^K \Bigg| \sum_{j=1}^K \vv_j = 0 \right\} \subseteq \Vcal_K.
  \end{equation*}
  Now we assume that random vector $\rvv = (\ervv_1, \dots, \ervv_K)$ is uniformly distributed on $\Bcal_K$, then for any distinct pair $j \ne c$, we have that:
  \begin{equation} \label{eq::pos-condition-expect-for-rv-balance}
    \Expect_\rvv \left[ \ervv_j | \ervv_c = +1 \right] \overset{(a)}{=} \frac{\left(\frac{K}{2} - 1\right) - \left(\frac{K}{2}\right)}{K-1} = - \frac{1}{K-1},
  \end{equation}
  where step $(a)$ follows from the fact that, condition on $\ervv_c = +1$, the remaining $K-1$ coordinates are symmetric, and $\frac{K}{2} - 1$ of them are $+1$, $\frac{K}{2}$ of them are $-1$. Similarly, we can prove that:
  \begin{equation} \label{eq::neg-condition-expect-for-rv-balance}
    \Expect_\rvv \left[ \ervv_j | \ervv_c = -1 \right] = \frac{\left(\frac{K}{2}\right) - \left(\frac{K}{2}- 1\right)}{K-1} = \frac{1}{K-1}.
  \end{equation}
  For fixed $\mW, \mY$, we define
  \begin{equation*}
    Z_i(\vv) = \sum_{j=1}^K \mW_{ij} \mY_{ij} \vv_j
  \end{equation*}
  for each row $i$. Then we have that
  \begin{equation*}
    \Expect_\rvv[Z_i(\rvv) | \ervv_{c_i} = +1] = \mW_{i c_i} - \sum_{j \ne c_i} \mW_{ij} \Expect[\ervv_j | \ervv_{c_i} = +1] \overset{(a)}{=} \mW_{i c_i} + \frac{1}{K-1} \sum_{j \ne c_i} \mW_{ij},
  \end{equation*}
  where step $(a)$ is from \Cref{eq::pos-condition-expect-for-rv-balance}. Since $\mW_{i c_i} = s_i a_i$ and $\sum_{j\ne c_i}^K \mW_{ij} = s_i (1 - a_i)$, we have:
  \begin{equation*}
    \Expect_\rvv[Z_i(\rvv) | \ervv_{c_i} = +1] = s_i \left( a_i + \frac{1 - a_i}{K-1} \right).
  \end{equation*}
  Similarly, we have
  \begin{equation*}
    \begin{aligned}
      \Expect_\rvv[Z_i(\rvv) | \ervv_{c_i} = -1] &= -\mW_{i c_i} - \sum_{j \ne c_i} \mW_{ij} \Expect[\ervv_j | \ervv_{c_i} = -1] \\
      &= -\mW_{i c_i} - \frac{1}{K-1} \sum_{j \ne c_i} \mW_{ij} = -s_i \left( a_i + \frac{1 - a_i}{K-1} \right).
    \end{aligned}
  \end{equation*}
  By the distribution of $\rvv$, it is obvious that $\Prob[\ervv_{c_i} = +1] = \Prob[\ervv_{c_i} = -1] = \frac{1}{2}$. Collecting the above together, we have:
  \begin{equation*}
    \begin{aligned}
      \Expect_\rvv \left[ |Z_i(\rvv)| \right] &= \Prob[\ervv_{c_i} = +1] \cdot \Expect_\rvv \left[ |Z_i(\rvv)| \big| \ervv_{c_i} = +1\right] + \Prob[\ervv_{c_i} = -1] \cdot \Expect_\rvv \left[ |Z_i(\rvv)| \big| \ervv_{c_i} = -1\right] \\
      &= \frac{1}{2} \Expect_\rvv \left[ |Z_i(\rvv)| \big| \ervv_{c_i} = +1\right] + \frac{1}{2} \Expect_\rvv \left[ |Z_i(\rvv)| \big| \ervv_{c_i} = -1\right] \\
      &\ge \frac{1}{2} \Expect_\rvv[Z_i(\rvv) | \ervv_{c_i} = +1] + \frac{1}{2} \Expect_\rvv[-Z_i(\rvv) | \ervv_{c_i} = -1] \\
      &= \frac{1}{2} s_i \left( a_i + \frac{1 - a_i}{K-1} \right) + \frac{1}{2} s_i \left( a_i + \frac{1 - a_i}{K-1} \right) = s_i \left( a_i + \frac{1 - a_i}{K-1} \right).
    \end{aligned}
  \end{equation*}
  Then we get that
  \begin{equation*}
    \Expect_\rvv [\wSig(\mW, \mY, \rvv)] = \sum_{i=1}^n \Expect_\rvv \left[ |Z_i(\rvv)| \right] \ge \sum_{i=1}^n s_i \left( a_i + \frac{1 - a_i}{K-1} \right) = \rho + \frac{1-\rho}{K-1}.
  \end{equation*}
  So
  \begin{equation*}
    \underset{\vv \in \Vcal_K}{\max}\wSig(\mW,\mY,\vv) \ge \Expect_\rvv [\wSig(\mW, \mY, \rvv)] \ge \rho + \frac{1-\rho}{K-1}.
  \end{equation*}
\end{proof}

\begin{customlemma}{\ref{lem::near-balance-even-lower-bound}}
  Fix integers $n \ge 1$ and $K \ge 3$, when $K$ is \textbf{odd}, then for all $\mW \in \Wcal_{n,K}$ and all $\mY \in \Ycal_{n,K}$, we have that:
  \begin{equation*}
    \underset{\vv \in \Vcal_K}{\max}\wSig(\mW,\mY,\vv) \ge \rho + \frac{1-\rho}{K}.
  \end{equation*}
\end{customlemma}

\begin{proof}[Proof of Lemma \ref{lem::near-balance-even-lower-bound}]
  When $K \ge 3$ and $K$ is odd, there is no balanced vectors, we then define a set of nearly balanced vectors as:
  \begin{equation*}
    \Bcal_K^\prime \coloneqq \left\{ \vv \in \PNOne^K \Bigg| \left| \sum_{j=1}^n \vv_j \right| = 1 \right\}.
  \end{equation*}
  Now we assume that random vector $\rvv = (\ervv_1, \dots, \ervv_K)$ is uniformly distributed on $\Bcal_K^\prime$. Since the elements of $\rvv$ are symmetric, and it is easy to see that $-\rvv, \rvv$ have the same distribution, we can get that:
  \begin{equation*}
    \Prob(\ervv_j = +1) = \Prob(\ervv_j = -1) = \frac{1}{2} \quad \forall j \in [K].
  \end{equation*}
  Since $\rvv \in \Bcal_K^\prime$, we have $\left( \sum_{j=1}^K \ervv_j \right)^2 = 1$, so $\Expect_\rvv \left[\left( \sum_{j=1}^K \ervv_j \right)^2 \right]= 1$. Expanding the square, we have
  \begin{equation*}
    \Expect_\rvv \left[\left( \sum_{j=1}^K \ervv_j \right)^2 \right] = \sum_{j=1}^K \Expect_\rvv [\ervv_j^2] + 2 \sum_{1 \le j < k \le K} \Expect_\rvv [\ervv_j \ervv_k]
  \end{equation*}
  By the symmetry of all the coordinates of $\rvv$, there exists a constant $r \in [-1,1]$ such that for all $j \ne k$, we have $\Expect_\rvv [\ervv_j \ervv_k] = r$. So we have that:
  \begin{equation*}
    1 = \Expect_\rvv \left[\left( \sum_{j=1}^K \ervv_j \right)^2 \right] = \sum_{j=1}^K \Expect_\rvv [\ervv_j^2] + K(K-1) r \overset{(a)}{=} K + K(K-1) r,
  \end{equation*}
  where $(a)$ is because $\Expect_\rvv [\ervv_j^2] = 1$ for all $j \in [K]$. Solving the above equation, we can get that $\Expect_\rvv [\ervv_j \ervv_k] = r = -\frac{1}{K}$ for all $j \ne k$. Then, for $j \ne k$, by the fact that
  \begin{equation*}
    \begin{aligned}
      \Expect_\rvv [\ervv_j \ervv_k] &= \Prob(\ervv_k = +1) \cdot \Expect_\rvv [\ervv_j \ervv_k | \ervv_k = +1] + \Prob(\ervv_k = -1) \cdot \Expect_\rvv [\ervv_j \ervv_k | \ervv_k = -1] \\
      &= \Prob(\ervv_k = +1) \cdot \Expect_\rvv [\ervv_j | \ervv_k = +1] - \Prob(\ervv_k = -1) \cdot \Expect_\rvv [\ervv_j | \ervv_k = -1] \\
      &= \frac{1}{2} \Expect_\rvv [\ervv_j | \ervv_k = +1] - \frac{1}{2} \Expect_\rvv [\ervv_j | \ervv_k = -1] \\
      &= \frac{1}{2} \Expect_\rvv [\ervv_j | \ervv_k = +1] + \frac{1}{2} \Expect_\rvv [-\ervv_j | -\ervv_k = +1] \\
      &\overset{(a)}{=} \Expect_\rvv [\ervv_j | \ervv_k = +1] \overset{(b)}{=} -\Expect_\rvv [\ervv_j | \ervv_k = -1],
    \end{aligned}
  \end{equation*}
  where steps $(a)$ and $(b)$ follow from the fact that $\rvv$ and $-\rvv$ have the same distribution. So we have that
  \begin{equation} \label{eq::condition-expects-for-rv-nearly-balance}
    \Expect_\rvv [\ervv_j | \ervv_k = +1] = - \frac{1}{K}, \quad \Expect_\rvv [\ervv_j | \ervv_k = -1] = \frac{1}{K}.
  \end{equation}
  Then by similar arguments shown in the proof of Lemma \ref{lem::balance-even-lower-bound}, we have
  \begin{equation*}
    \begin{aligned}
      \Expect_\rvv[Z_i(\rvv) | \ervv_{c_i} \!=\! +1] &= \mW_{i c_i} \!-\! \sum_{j \ne c_i} \mW_{ij} \Expect[\ervv_j | \ervv_{c_i} \!=\! +1] = \mW_{i c_i} \!+\! \frac{1}{K} \sum_{j \ne c_i} \mW_{ij} \!=\! s_i \left( a_i \!+\! \frac{1-a_i}{K}\right), \\
      \Expect_\rvv[Z_i(\rvv) | \ervv_{c_i} \!=\! -1] &= -\mW_{i c_i} \!-\! \sum_{j \ne c_i} \!\mW_{ij} \Expect[\ervv_j | \ervv_{c_i} \!=\! -1] \!=\! -\mW_{i c_i} \!-\! \frac{1}{K} \!\sum_{j \ne c_i} \!\mW_{ij} \!=\! -s_i \!\left(\! a_i \!+\! \frac{1 \!-\! a_i}{K} \!\right).
    \end{aligned}
  \end{equation*}
  By similar arguments in the proof of Lemma \ref{lem::balance-even-lower-bound} again, we can obtain that
  \begin{equation*}
    \begin{aligned}
      \Expect_\rvv \left[ |Z_i(\rvv)| \right] &\ge \frac{1}{2} \Expect_\rvv[Z_i(\rvv) | \ervv_{c_i} = +1] + \frac{1}{2} \Expect_\rvv[-Z_i(\rvv) | \ervv_{c_i} = -1] \\
      &= \frac{1}{2} s_i \left( a_i + \frac{1 - a_i}{K} \right) + \frac{1}{2} s_i \left( a_i + \frac{1 - a_i}{K} \right) = s_i \left( a_i + \frac{1 - a_i}{K} \right).
    \end{aligned}
  \end{equation*}
  Then
  \begin{equation*}
    \underset{\vv \in \Vcal_K}{\max}\wSig(\mW,\mY,\vv) \!\ge\! \Expect_\rvv [\wSig(\mW, \mY, \rvv)] \!=\! \sum_{i=1}^n \Expect_\rvv \left[ |Z_i(\rvv)| \right] \ge \sum_{i=1}^n s_i \left( a_i + \frac{1 - a_i}{K} \right) = \rho + \frac{1-\rho}{K}.
  \end{equation*}
\end{proof}

\subsubsection{Lower Bound for Even $K$} \label{subsec::lower-bound-even-K}
Putting Lemmas~\ref{lem::all-one-lower-bound} and \ref{lem::balance-even-lower-bound} together, we know that when $K$ is even and $K \ge 2$,
\begin{equation} \label{ieq::lower-bound-even-rho}
  \underset{\vv \in \Vcal_K}{\max}\wSig(\mW,\mY,\vv) \ge \max \left\{ |2\rho-1|, \rho + \frac{1-\rho}{K-1} \right\},
\end{equation}
where $\rho \in [0,1]$ and depends on $\mW, \mY$. Minimize the right side of \Cref{ieq::lower-bound-even-rho} over $\rho \in [0,1]$, then for any $\mW \in \Wcal_{n,K}, \mY \in \Ycal_{n,K}$,
\begin{equation*}
  \underset{\vv \in \Vcal_K}{\max}\wSig(\mW,\mY,\vv) \ge \underset{\rho \in [0,1]}{\min} \max\left\{ |2\rho-1|, \rho + \frac{1-\rho}{K-1} \right\}.
\end{equation*}
If $2\rho - 1 \le 0$, i.e., $0 \le \rho \le \frac{1}{2}$, then $|2\rho -1| = 1-2\rho$, which decreases with $\rho$. It is easy to see that $\rho + \frac{1-\rho}{K-1} = \frac{K-2}{K-1} \rho + \frac{1}{K-1}$, which increases with $\rho$. So when $1-2\rho = \rho + \frac{1-\rho}{K-1}$, we achieve $\underset{\rho \in \left[0,\tfrac{1}{2}\right]}{\min} \max\left\{ |2\rho-1|, \rho + \frac{1-\rho}{K-1} \right\}$. Solving the equation $1-2\rho = \rho + \frac{1-\rho}{K-1}$ yields
\begin{equation*}
  \rho_{\mathrm{even}}^* = \frac{K-2}{3K-4} < \frac{1}{3}.
\end{equation*}
The corresponding minimal value is
\begin{equation*}
  1 - 2 \rho_{\mathrm{even}}^* = \frac{K}{3K-4}.
\end{equation*}

If $\rho \ge \frac{1}{2}$, we have:
\begin{equation*}
  \begin{aligned}
    &\underset{\rho \in \left[\frac{1}{2},1\right]}{\min} \max\left\{ |2\rho-1|, \rho + \frac{1-\rho}{K-1} \right\} \ge \underset{\rho \in \left[\frac{1}{2},1\right]}{\min} \left(\rho + \frac{1-\rho}{K-1} \right) \\
    &\overset{(a)}{=} \rho + \frac{1-\rho}{K-1} \bigg|_{\rho=\tfrac{1}{2}} = \frac{K}{2K-2} \overset{(b)}{\ge} \frac{K}{3K-4},
  \end{aligned}
\end{equation*}
where step $(a)$ is from the fact that $\rho + \frac{1-\rho}{K-1}$ increases with $\rho$; and step $(b)$ is because $K \ge 2$. Putting all the above together, we get that when $K \ge 2$ and is even, for any $\mW \in \Wcal_{n,K}, \mY \in \Ycal_{n,K}$,
\begin{equation*}
  \underset{\vv \in \Vcal_K}{\max}\wSig(\mW,\mY,\vv) \ge \frac{K}{3K-4}.
\end{equation*}
Minimizing over $\mW \in \Wcal_{n,K}, \mY \in \Ycal_{n,K}$, we have
\begin{equation*}
  \Wfrak_{n,K} \ge \frac{K}{3K-4} = C_K.
\end{equation*}

\subsubsection{Lower Bound for Odd $K$}

In this section, we take similar steps as Section \ref{subsec::lower-bound-even-K} to prove the result. Putting Lemmas~\ref{lem::all-one-lower-bound} and \ref{lem::near-balance-even-lower-bound} together, we know that when $K$ is odd and $K \ge 3$,
\begin{equation} \label{ieq::lower-bound-odd-rho}
  \underset{\vv \in \Vcal_K}{\max}\wSig(\mW,\mY,\vv) \ge \max \left\{ |2\rho-1|, \rho + \frac{1-\rho}{K} \right\},
\end{equation}
where $\rho \in [0,1]$ and depends on $\mW, \mY$. Minimize the right side of \Cref{ieq::lower-bound-odd-rho} over $\rho \in [0,1]$, then for any $\mW \in \Wcal_{n,K}, \mY \in \Ycal_{n,K}$,
\begin{equation*}
  \underset{\vv \in \Vcal_K}{\max}\wSig(\mW,\mY,\vv) \ge \underset{\rho \in [0,1]}{\min} \max\left\{ |2\rho-1|, \rho + \frac{1-\rho}{K} \right\}.
\end{equation*}
When $0 \le \rho \le \frac{1}{2}$, $|2\rho -1| = 1 - 2 \rho$ decreases with $\rho$, and $\rho + \frac{1-\rho}{K} = \frac{K-1}{K} \rho + \frac{1}{K}$ increases as $\rho$. So when $1-2\rho = \rho + \frac{1-\rho}{K}$, we achieve $\underset{\rho \in \left[0,\tfrac{1}{2}\right]}{\min} \max\left\{ |2\rho-1|, \rho + \frac{1-\rho}{K} \right\}$. Solving the equation $1-2\rho = \rho + \frac{1-\rho}{K}$ yields
\begin{equation*}
  \rho_\mathrm{odd}^* = \frac{K-1}{3K-1} < \frac{1}{3}.
\end{equation*}
The corresponding minimal value is
\begin{equation*}
  1 - 2\rho_\mathrm{odd}^* = \frac{K+1}{3K-1}.
\end{equation*}
When $\rho \ge \frac{1}{2}$,
\begin{equation*}
  \begin{aligned}
    &\underset{\rho \in \left[\frac{1}{2},1\right]}{\min} \max\left\{ |2\rho-1|, \rho + \frac{1-\rho}{K} \right\} \ge \underset{\rho \in \left[\frac{1}{2},1\right]}{\min} \left(\rho + \frac{1-\rho}{K} \right) \\
    &\overset{(a)}{=} \rho + \frac{1-\rho}{K} \bigg|_{\rho=\tfrac{1}{2}} = \frac{K+1}{2K} \overset{(b)}{\ge} \frac{K+1}{3K-1},
  \end{aligned}
\end{equation*}
where step $(a)$ is from the fact that $\rho + \frac{1-\rho}{K}$ increases with $\rho$; and step $(b)$ is because $K \ge 3$. Putting all the above together, we get that when $K \ge 3$ and is odd, for any $\mW \in \Wcal_{n,K}, \mY \in \Ycal_{n,K}$,
\begin{equation*}
  \underset{\vv \in \Vcal_K}{\max}\wSig(\mW,\mY,\vv) \ge \frac{K+1}{3K-1}.
\end{equation*}
Minimizing over $\mW \in \Wcal_{n,K}, \mY \in \Ycal_{n,K}$, we have

\begin{equation*}
  \Wfrak_{n,K} \ge \frac{K+1}{3K-1} = C_K.
\end{equation*}

\subsection{Proof of the Refined Lower Bound} \label{subsec::proof-lower-bound-refined}

In this section, we prove the refined lower bound, i.e., \Cref{thm::refine-lower-bound}. The proof of \Cref{thm::refine-lower-bound} is based on the lower bound in \Cref{thm::main-theorem}.

For given $\mW \in \Wcal_{n,K}, \mY \in \Ycal_{n,K}$, we define $c_i$ as the true label for the $i$-th example, i.e., $\mY_{i, c_i} = 1$ for all $i \in [n]$. We define
\begin{equation*}
  I_+ \coloneqq \left\{ i \in [n] : \sum_{j=1}^K \mW_{ij} > 0 \right\}
\end{equation*}
as the set of indices whose row-weight summation is non-zero. We only need to consider the rows with indices in $I_+$ because the behavior of the rows outside of $I_+$ does not affect the value of $\underset{\vv \in \Vcal_K}{\max}\wSig(\mW,\mY,\vv)$. Define the set of all labels as
\begin{equation*}
  S \coloneqq \{ c_i : i \in I_+ \}.
\end{equation*}

Let $r = |S| \le K$, by the above definitions, it is easy to get that $r \le |I_+| \le n$. If $r = K$, then $K = r \le n < n + 1$, then $\min\{n+1, K\} = K$. Then by the lower bound in \Cref{thm::main-theorem}, we know that
\begin{equation*}
  \Wfrak_{n,K} \ge C_K = C_{\min\{ n+1, K \}}.
\end{equation*}

Then we only need to consider the case that $r < K$. When $r < K$, let $q = r+1$. We enumerate the elements of $S$ and define
\begin{equation*}
  S = \{ s_1, \dots, s_r \}, \ \ \ \ \ \  U \coloneqq [K] \backslash S \ne \emptyset,
\end{equation*}
where $U \ne \emptyset$ because $r = |S| < K$ by our assumption.

For the given $\mW,\mY$, we then construct $\bar{\mW} \in \Wcal_{n,q}$ and $\bar{\mY} \in \Ycal_{n,q}$ so that the following inequation holds:
\begin{equation*}
  \underset{\vv \in \Vcal_K}{\max}\wSig(\mW,\mY,\vv) \ge \underset{\vv \in \Vcal_q}{\max}\wSig(\bar{\mW},\bar{\mY},\vv).
\end{equation*}

For all $i \in [n]$, define
\begin{equation*}
  \bar{\mW}_{ij} \coloneqq \mW_{i s_j}\  \mathrm{for} \ 1 \le j \le r,\ \ \ \ 
  \bar{\mW}_{iq} \coloneqq \sum_{j \in U} \mW_{ij},
\end{equation*}
i.e., we put the weights of the rows inside $U$ in the $q$-th row of $\bar{\mW}$. It is obvious that $\bar{\mW}_{ij} \ge 0$ for all $i \in [n], j \in [q]$. Moreover,
\begin{equation*}
  \begin{aligned}
    \sum_{i=1}^n \sum_{j=1}^q \bar{\mW}_{ij} &= \sum_{i=1}^n \left( \sum_{j=1}^r \mW_{i s_j} + \sum_{j \in U} \mW_{ij} \right) = \sum_{i=1}^n \sum_{j=1}^K \mW_{ij} = 1.
  \end{aligned}
\end{equation*}
So we verified that $\bar{\mW} \in \Wcal_{n,q}$. Then we construct $\bar{\mY}$.

For any $i \in I_+$, we have $c_i \in S$, then by the definition of $S$, there exists an unique index $l(i) \in [r]$ such that $c_i = s_{l(i)}$. Then for all $i \in I_+$, we define
\begin{equation*}
  \bar{\mY}_{i,l(i)} = +1, \ \ \ \  \bar{\mY}_{ij} = -1 \ \mathrm{for}\  j \in [r+1]\backslash \{l(i) \}.
\end{equation*}
For $i \notin I_+$, the behavior of the $i$-th row does not affect the result, we can set $\bar{\mY}_i$ to be any $q$-dimensional signed one-hot vector. It is easy to see that $\bar{\mY} \in \Ycal_{n,q}$.

Fix the vector $\vu = (\evu_1, \dots, \evu_q) \in \Vcal_q$, we construct a $K$-dimensional vector $\vv \in \Vcal_K$ as follows:
\begin{equation*}
  \vv_{s_j} (\vu) = \evu_j \ \mathrm{for} \ 1 \le j \le r, \ \ \  \vv_j(\vu) = \evu_q \ \mathrm{for} \ j \in U.
\end{equation*}
For any $i \in I_+$, since the true label is $c_i$, then $\mY_{i c_i} = +1$ and $\mY_{ij} = -1$ for $j \in U$. Then we have:
\begin{equation*}
  \begin{aligned}
    \sum_{j=1}^K \mW_{ij} \mY_{ij} \vv_j(\vu) &= \mW_{i c_i} \vv_{c_i}(\vu) - \sum_{1 \le j \le K, j \ne c_i} \mW_{ij} \vv_j(\vu) \\
    &\overset{(a)}{=} \mW_{i c_i} \vu_{l(i)} - \sum_{1 \le j \le K, j \ne c_i} \mW_{ij} \vv_j(\vu) \\
    &= \mW_{i c_i} \vu_{l(i)} - \sum_{j \in S\backslash \{ c_i\}} \mW_{ij} \vv_j(\vu)- \sum_{j \in U} \mW_{ij} \vv_j(\vu) \\
    &\overset{(b)}{=} \mW_{i c_i} \vu_{l(i)} - \sum_{1 \le l \le r, l \ne l(i)} \mW_{i s_l} \vv_{s_l}(\vu)- \sum_{j \in U} \mW_{ij} \vv_j(\vu) \\
    &\overset{(c)}{=} \mW_{i s_{l(i)}} \vu_{l(i)} - \sum_{1 \le l \le r, l \ne l(i)} \mW_{i s_l} \vu_l - \sum_{j \in U} \mW_{ij} \vu_q \\
    &\overset{(d)}{=} \bar{\mW}_{i l(i)} \bar{\mY}_{i l(i)} \vu_{l(i)} + \sum_{1 \le l \le r, l \ne l(i)} \bar{\mW}_{i l} \bar{\mY}_{i l} \vu_l + \bar{\mW}_{iq} \bar{\mY}_{iq} \vu_q \\
    &= \sum_{l=1}^q \bar{\mW}_{il} \bar{\mY}_{il} \vu_l,
  \end{aligned}
\end{equation*}

where step $(a)$ is from the definition of $\vv(\vu)$ and the fact that $c_i = s_{l(i)}$; $(b)$ changes the indices of the sum over $S \backslash \{c_i \}$; $(c)$ is from the definition of $\vv_{s_j} (\vu)$ and the fact that $c_i = s_{l(i)}$; and $(d)$ is from the definition of $\bar{\mW}, \bar{\mY}$. If $i \notin I_+$, both sides of the above equation equals to zero. So for any $i \in [n]$, we have:
\begin{equation*}
  \left| \sum_{j=1}^K \mW_{ij} \mY_{ij} \vv_j(\vu) \right|= \left|\sum_{l=1}^q \bar{\mW}_{il} \bar{\mY}_{il} \vu_l \right|.
\end{equation*}

Sum over $i \in [n]$ yields
\begin{equation*}
  \wSig(\mW,\mY,\vv(\vu)) = \wSig(\bar{\mW}, \bar{\mY}, \vu).
\end{equation*}
Assume $\vu^* = \arg\max_{\vu \in \Vcal_q} \wSig(\bar{\mW}, \bar{\mY}, \vu)$, then we have:
\begin{equation*}
  \underset{\vv \in \Vcal_K}{\max} \wSig(\mW,\mY,\vv) \ge \wSig(\mW,\mY,\vv(\vu^*)) = \wSig(\bar{\mW}, \bar{\mY}, \vu^*) = \underset{\vu \in \Vcal_q}{\max} \wSig(\bar{\mW}, \bar{\mY}, \vu).
\end{equation*}

By the proof of the lower bounds of \Cref{thm::main-theorem}, we have $\underset{\vu \in \Vcal_q}{\max} \wSig(\bar{\mW}, \bar{\mY}, \vu) \ge C_q$. So we can conclude that when $r < K$,
\begin{equation*}
  \underset{\vv \in \Vcal_K}{\max} \wSig(\mW,\mY,\vv) \ge C_q \overset{(a)}{\ge} C_{\min \{ n+1, K \}},
\end{equation*}
where step $(a)$ is from the fact that $q \le \min\{n+1,K \}$(because $q = r + 1 \le n+1$ and $q = r + 1 < K + 1$) and the monotonicity property in Proposition~\ref{prop::Cq-properties}. Minimize over $\mW \in \Wcal_{n,K}, \mY \in \Ycal_{n,K}$, we get
\begin{equation*}
  \Wfrak_{n, K} \ge C_{\min \{ n+1, K \}}.
\end{equation*}
Putting the cases $r < K$ and $r \ge K$ together, we conclude that for any integers $n \ge 1$ and $K \ge 2$,
\begin{equation*}
  \Wfrak_{n, K} \ge C_{\min \{ n+1, K \}}.
\end{equation*}

\section{Proof of the Upper Bounds} \label{sec::proof-upper-bound}

\subsection{Proof of the Upper Bound in \Cref{thm::main-theorem}} \label{subsec::proof-upper-bound-raw}

Let $M \coloneqq \min\{ n,K\}$, we choose any $\mY^\prime \in \Ycal_{n,K}$ such that:
\begin{equation*}
  \mY_{ii}^\prime = +1 \quad\quad \forall i \in [M],
\end{equation*}
i.e., we set the true label of the $i$-th example as $i$ for $i \in [M]$. For any fixed parameter $a \in [0,1]$, we construct the weight matrix $\mW^a$ as follows:
\begin{enumerate}
  \item We set $\mW^a_{ii} \coloneqq \frac{a}{M}$ for all $i \in [M]$.
  \item For $i \in [M]$ and $j \in [M]\backslash \{ i \}$, we set $\mW^a_{ij} \coloneqq \frac{\frac{1}{M} - \mW^a_{ii}}{M-1} = \frac{1-a}{M(M-1)}$.
  \item When $i > M$ or $j > M$, we set $\mW^a_{ij} = 0$.
\end{enumerate}
From the above construction, it is easy to verify that $\mW^a \in \Wcal_{n,K}$ and $s_i = \sum_{j=1}^M \mW^a_{ij} = \frac{1}{M}$. When $M = 1$, it is easy to see that $\mW^a_{11} = a$ and $\wSig(\mW^a,\mY^\prime,\vv) = 1$ for any $\vv \in \Vcal_K$, which directly leads to the upper bound $\Wfrak_{n,K} \le C_1$. Without loss of generality, we assume that $M \ge 2$ in the remaining arguments.

Next, we calculate $\wSig(\mW^a,\mY^\prime,\vv)$ for a given $\vv$. Fix $\vv \in \Vcal_K$, since $\vv_{M+1}, \dots, \vv_K$ do not effect the value of $\wSig(\mW^a,\mY^\prime,\vv)$, define
\begin{equation*}
  V \coloneqq \sum_{j=1}^M \vv_j, \quad b \coloneqq \frac{1-a}{M-1}.
\end{equation*}
Then
\begin{equation*}
  \begin{aligned}
    \wSig(\mW^a,\mY^\prime,\vv) &= \sum_{i=1}^n \left| \sum_{j=1}^K \mW^a_{ij} \cdot \mY^\prime \cdot \vv_{j} \right| \overset{(a)}{=} \sum_{i=1}^M \left| \mW^a_{ii} \vv_i - \sum_{j=1,j\ne i}^M \mW^a_{ij} \vv_{j} \right| \\
    &\overset{(b)}{=} \frac{1}{M} \sum_{i=1}^M \left| a \vv_i - \sum_{j=1,j\ne i}^M b \vv_{j} \right| \overset{(c)}{=} \frac{1}{M} \sum_{i=1}^M \left| a \vv_i - b (V - \vv_i)  \right| \\
    &= \frac{1}{M} \sum_{i=1}^M \left| (a+b)\vv_i -bV  \right|,
  \end{aligned}
\end{equation*}
where steps $(a)$ and $(b)$ are from the definition of $\mW^a, \mY^\prime$; step $(c)$ is from the fact that $\sum_{j=1,j\ne i}^M \vv_j = V - \vv_i$. Assume that among $\vv_1, \dots, \vv_M$, there are $p$ of them have value $+1$, and $M - p$ of them have value $-1$. Then, solving the equation $p - (M - p) = V$ gives
\begin{equation*}
  p = \frac{M+V}{2}, \quad\quad M-p = \frac{M-V}{2}.
\end{equation*}
Then,
\begin{equation*}
  \wSig(\mW^a,\mY^\prime,\vv) = \frac{1}{M} \left[p |(a+b)-bV| + (M-p) |(a+b)+bV| \right] \coloneqq f(V),
\end{equation*}
where the possible value set of $V$ is
\begin{equation*}
  D_V \coloneqq \left\{ -M , -M + 2, \dots, M-2, M \right\}.
\end{equation*}
Fix $a,b,M$, the next step is to figure out the maximal value of $f(V)$ over $D_V$. By the definition of $f(V)$, it is easy to see that $f(V) = f(-V)$ for any $V \in D_V$, so it suffices to consider $0 \le V \le M$. Since $a,b,V \ge 0$, we have $f(V) = \frac{1}{M} \left[p |(a+b)-bV| + (M-p) \left((a+b)+bV \right) \right]$.

When $(a+b)-bV \ge 0$, i.e., $0 \le V \le \frac{a+b}{b}$, we have
\begin{equation*}
  f(V) \!=\! \frac{1}{M} \left[p \left((a+b)\!-\! bV \right) \!+\! (M\!-\!p) \left((a+b)+bV \right) \right] \!=\! (a+b)+ \frac{bV}{M} (M-2p) \overset{(a)}{=} (a+b) - \frac{bV^2}{M},
\end{equation*}
where step $(a)$ is because $p = \frac{M+V}{2}$.

When $(a+b)-bV \le 0$, i.e., $\frac{a+b}{b} \le V \le M$, we have
\begin{equation*}
  f(V) \!=\! \frac{1}{M} \left[p \left(bV \!-\! (a+b) \right) \!+\! (M \!-\! p) \left((a+b)+bV \right) \right] \!=\! bV + \frac{a+b}{M} (M-2p) = V \left( b - \frac{a+b}{M} \right).
\end{equation*}
Collecting the above, we get
\begin{equation} \label{eq::f-V}
  f(V) = \begin{cases}
  (a+b) - \frac{bV^2}{M} & \text{if } 0 \le V \le \frac{a+b}{b} \\
  V \left( b - \frac{a+b}{M} \right) & \text{if } \frac{a+b}{b} \le V \le M.
  \end{cases}
\end{equation}

\subsubsection{Upper Bound for Even $M$} \label{subsubsec::upper-bound-even-M}
When $M \ge 2$ and is even, take $a = a^e \coloneqq \frac{M-2}{3M-4}$, then we have
\begin{equation*}
  b = \frac{1-a}{M-1} = \frac{2}{3M-4}, \quad\quad a+b = \frac{M}{3M-4}, \quad\quad \frac{a+b}{b} = \frac{M}{2}.
\end{equation*}
According to \Cref{eq::f-V}, when $0 \le V \le \frac{M}{2}$,
\begin{equation*}
  f(V) = (a+b) - \frac{bV^2}{M} \le a + b = \frac{M}{3M-4};
\end{equation*}
when $\frac{M}{2} \le V \le M$,
\begin{equation*}
  b - \frac{a+b}{M} = \frac{1}{3M -4} > 0,
\end{equation*}
so we have
\begin{equation*}
  f(V) \le f\left( M \right) = M \cdot \frac{1}{3M -4} = \frac{M}{3M-4}.
\end{equation*}
In conclusion, $f(V) \le \frac{M}{3M-4}$ on $D_V$, which means that
\begin{equation*}
  \Wfrak_{n,K} \le \underset{\vv \in \Vcal_K}{\max}\wSig(\mW^{a^e},\mY^\prime,\vv) \le \frac{M}{3M-4}.
\end{equation*}

\subsubsection{Upper Bound for Odd $M$} \label{subsubsec::upper-bound-odd-M}
When $M \ge 3$ and is odd, take $a = a^o \coloneqq \frac{M-1}{3M-1}$, then we have
\begin{equation*}
  b = \frac{1-a}{M-1} = \frac{2M}{(M-1)(3M-1)}, \quad\quad a+b = \frac{M^2+1}{(M-1)(3M-1)}, \quad\quad \frac{a+b}{b} = \frac{M^2+1}{2M}.
\end{equation*}
Since $M$ is odd, $V$ can not have even value, then we consider $V \in \{ 1,3,5, \dots, M \}$. According to \Cref{eq::f-V}, when $1 \le V \le \frac{a+b}{b}$, $f(V)$ decreases with $V$, so we have $f(V) \le f(1) = \frac{M+1}{3M-1}$.

When $\frac{a+b}{b} \le V \le M$, 
\begin{equation*}
  b - \frac{a+b}{M} = \frac{M+1}{M(3M-1)} > 0,
\end{equation*}
so we have
\begin{equation*}
  f(V) \le f(M) = M \cdot \frac{M+1}{M(3M-1)} = \frac{M+1}{3M-1}.
\end{equation*}
In conclusion, $f(V) \le \frac{M+1}{3M-1}$ on $D_V$, which means that
\begin{equation*}
  \Wfrak_{n,K} \le \underset{\vv \in \Vcal_K}{\max}\wSig(\mW^{a^o},\mY^\prime,\vv) \le \frac{M+1}{3M-1}.
\end{equation*}

Putting the results in Sections \ref{subsubsec::upper-bound-even-M} and \ref{subsubsec::upper-bound-odd-M} together proves the upper bound.

\subsection{Proof of the Upper Bound in \Cref{thm::another-upper-bound}} \label{subsec::proof-upper-bound-another}

If $n$ is even, then $n+1$ is odd and $C_{n+1} = \frac{(n+1)+1}{3(n+1)-1} = \frac{n+2}{3n+2}$. For simplicity, define $D=3n+2$, then we construct $\mW \in \Wcal_{n,K}, \mY \in \Ycal_{n,K}$.

Since $n+1 \le K$, we define
\begin{equation*}
  \mY_{ii} = +1 \ \mathrm{for}\  i \in [n], \ \ \ \  \mY_{ij} = -1 \ \mathrm{for\ all} \ j \ne i.
\end{equation*}
It is easy to verify that $\mY \in \Ycal_{n,K}$.

For $\mW$, we set the weights for the $j$-th column to be $0$ for all $j \ge n+2$ and define
\begin{equation*}
  \begin{aligned}
    \mW_{ii} &= \frac{1}{D}\ \ \ \mathrm{for} \ \ i \in [n], \\
    \mW_{ij} &= \frac{2}{nD}\ \ \  \mathrm{for} \ \ 1 \le j \le n \ \mathrm{and} \ \ j \ne i, \\
    \mW_{i,n+1} &= \frac{4}{nD} \ \ \ \mathrm{for} \ \  i \in [n], \\
    \mW_{ij} &= 0 \ \ \ \mathrm{for}\ \ j \ge n+2.
  \end{aligned}
\end{equation*}
It is obvious that $\mW_{ij} \ge 0$ for $i \in [n], j \in [K]$. For any $i \in [n]$, we have:
\begin{equation*}
  \sum_{j=1}^K \mW_{ij} = \frac{1}{D} + (n-1) \frac{2}{nD} + \frac{4}{nD} = \frac{3n+2}{nD} = \frac{1}{n}.
\end{equation*}
So the weights for each row are equal and $\sum_{i=1}^n \sum_{j=1}^K \mW_{ij} = 1$, so $\mW \in \Wcal_{n,K}$.

For any $\vv \in \Vcal_K$, let
\begin{equation*}
  V \coloneqq \sum_{j=1}^n \evv_j, \ \ \ \ u = \evv_{n+1},
\end{equation*}
then we have:
\begin{equation*}
  \begin{aligned}
    Z_i(\vv) &\coloneqq \sum_{j=1}^K \mW_{ij} \mY_{ij} \evv_j =  \sum_{j=1}^{n+1} \mW_{ij} \mY_{ij} \evv_j \\
    &\overset{(a)}{=} \frac{1}{D} \evv_i - \frac{2}{nD} \sum_{1\le j \le n, j \ne i} \evv_j - \frac{4}{nD} u \\
    &\overset{(b)}{=} \frac{1}{D} \evv_i - \frac{2}{nD} (V-\evv_i) - \frac{4}{nD}u \\
    &= \frac{(n+1)\evv_i - 2V -4u}{nD},
  \end{aligned}
\end{equation*}
where $(a)$ is from the definitions of $\mW$ and $\mY$; and $(b)$ is from the definition of $V$. So
\begin{equation*}
  \wSig(\mW,\mY,\vv) = \sum_{i=1}^n |Z_i(\vv)| = \frac{1}{nD} \sum_{i=1}^n \left| (n+2)\evv_i -2V-4u \right|.
\end{equation*}

For simplicity, let $A = n+2$ and $B = 2V+4u$, let $p$ be the number of $+1$ in $\evv_1, \dots, \evv_n$, then by the definition of $V$, we have $p - (n-p) = V$, which means that
\begin{equation*}
  p = \frac{n+V}{2}.
\end{equation*}

Define the function
\begin{equation*}
  f(B) \coloneqq \sum_{i=1}^n |A\evv_i -B| = p |A - B| + (n-p) |A+B|,
\end{equation*}
where the last equation is from the definition of $p$. As functions of $\vv$, we can easily show that $V(-\vv) = -V(\vv)$, $B(-\vv) = - B(\vv)$, and that $f(B(-\vv)) = f(B(\vv))$. To get the maximal value of $f(B(\vv))$, we can assume with not loss of generality that $B \ge 0$. Since $A = n+2$, then we have $A+B > 0$, so
\begin{equation*}
  f(B) = p |A - B| + (n-p) (A+B), \ \ \ \  B \ge 0.
\end{equation*}

If $0 \le B \le A$, then $|A-B| = (A-B)$, then
\begin{equation*}
  f(B) = p(A-B) + (n-p)(A+B) = nA + (n-2p)B.
\end{equation*}
Since $2p = n+V$, then $n-2p = -V$, then we have:
\begin{equation*}
  f(B) = nA - VB.
\end{equation*}
Since $n$ is even, then by the definition of $V$, it is obvious that $V$ is even. By the definition of $B$ and the assumption that $B \ge 0$, we have $B = 2(V+2u) \ge 0$, i.e., $V+2u \ge 0$. We consider the value of $u$.
\begin{itemize}
  \item If $u = +1$, then we have $V+2 \ge 0$, so $V \ge -2$. Since $V$ is even, then either $V=-2$ or $V\ge 0$. When $V = -2$, then $B = 0$, so $BV = 0$. When $V \ge 0$, since $B \ge 0$, we have $BV \ge 0$. So we have $BV \ge 0$.
  \item If $u = -1$, then we have $V-2 \ge 0$, i.e., $V \ge 2$. Since $B \ge 0$, we have $BV \ge 0$.
\end{itemize}
So $BV \ge 0$ regardless of the value of $u$. Then we have:
\begin{equation*}
  f(B) = nA - VB \le nA = n(n+2).
\end{equation*}

If $B \ge A$, then
\begin{equation*}
  f(B) = p(B-A) + (n-p)(A+B) = nB + (n-2p)A = nB - VA.
\end{equation*}
Since $B = 2V + 4u$ and $A = n+2$, we can obtain that:
\begin{equation*}
  f(B) = n(2V+4u) - (n+2)V = (n-2)V + 4nu.
\end{equation*}
Since $u \le 1$ and $V \le n$, we have:
\begin{equation*}
  f(B) = (n-2)V + 4nu \le n(n-2) + 4n = n(n+2).
\end{equation*}

So we have $f(B) \le n(n+2)$ for $B\ge 0$. Then
\begin{equation*}
  \wSig(\mW,\mY,\vv) = \frac{f(B(\vv))}{nD} \le \frac{n(n+2)}{nD} = \frac{n+2}{3n+2} = C_{n+1}.
\end{equation*}

Minimize over $\mW \in \Wcal, \mY \in \Ycal$, we have:
\begin{equation*}
  \Wfrak_{n,K} \le C_{n+1}.
\end{equation*}

\section{Other Missing Proofs} \label{sec::proof-others}
\begin{proof}[Proof of Proposition \ref{prop::Cq-properties}]
  It is easy to verify that $C_1 = C_2 = 1$, so we consider $q \ge 2$.

  If $q$ is even, then
  \begin{equation*}
    \begin{aligned}
      C_{q+1} - C_q &= \frac{(q+1)+1}{3(q+1)-1} - \frac{q}{3q-4} = \frac{q+2}{3q+2} - \frac{q}{3q-4} \\
      &= \frac{(q+2)(3q-4) - q(3q+2)}{(3q+2)(3q-4)} = \frac{-8}{(3q+2)(3q-4)} < 0.
    \end{aligned}
  \end{equation*}
  If $q$ is odd, then
  \begin{equation*}
    C_{q+1} - C_q = \frac{(q+1)}{3(q+1)-4} - \frac{q+1}{3q-1} = \frac{q+1}{3q-1} - \frac{q+1}{3q-1} = 0.
  \end{equation*}
  So $C_{q+1} - C_q \le 0$ for all $q \in \mathbb{N}_+$, which means that $C_q$ is non-increasing with respect to $q$.

  It is easy to see that $\lim_{q \to \infty} \frac{q}{3q-4} = \frac{1}{3}$ and $\lim_{q \to \infty} \frac{q+1}{3q-1} = \frac{1}{3}$, so $\lim_{q \to \infty} C_q = \frac{1}{3}$.
\end{proof}

\begin{proof}[Proof of \Cref{thm::exact-value}]
  Define $M = \min\{ n + 1, K \}$, according to \Cref{thm::refine-lower-bound}, we know that $\Wfrak_{n,K} \ge C_M$. For the upper bound, we consider there cases.
  \begin{itemize}
    \item When $K \le n$, then $M = \min\{n+1, K \} = K = \min\{ n,K\}$. Then by the upper bound in \Cref{thm::main-theorem}, we know that $\Wfrak_{n,K} \le C_{\min\{ n,K\}} = C_M$.
    \item When $K \ge n+1$ and $n$ is odd, $M = \min\{n+1, K \} = n+1$ and $\min\{ n, K\} = n$. Then by the upper bound in \Cref{thm::main-theorem}, we have $\Wfrak_{n,K} \le C_{\min\{ n,K\}} = C_n $. By the proof of Proposition \ref{prop::Cq-properties}, when $n$ is odd, $C_{n+1} = C_n$. So $\Wfrak_{n,K} \le C_n = C_{n+1} = C_M$.
    \item When $K \ge n+1$ and $n$ is even, $M = \min\{n+1, K \} = n+1$, by \Cref{thm::another-upper-bound}, $\Wfrak_{n,K} \le C_{n+1} = C_M$.
  \end{itemize}
  In conclusion, for all integers $n \ge 1$ and $K \ge 2$, we have $\Wfrak_{n,K} = C_{\min\{ n+1, K\}}$.
\end{proof}

\end{document}